\documentclass[10pt,twocolumn,letterpaper]{article}

\usepackage[pagenumbers]{cvpr} 

\definecolor{cvprblue}{rgb}{0.21,0.49,0.74}
\usepackage[pagebackref,breaklinks,colorlinks,allcolors=cvprblue]{hyperref}

\usepackage{graphicx}
\usepackage{amsmath}
\usepackage{amssymb}
\usepackage{amsthm}
\usepackage{booktabs}
\usepackage{algorithm}
\usepackage{algorithmic}
\usepackage{multirow}
\usepackage{xcolor}
\usepackage{tikz}
\usepackage{pgfplots}
\usetikzlibrary{positioning,shapes.geometric}
\usepackage{tabularx,array} 
\newcolumntype{Y}{>{\raggedright\arraybackslash}X} 
\newcolumntype{L}{>{\raggedright\arraybackslash}p{0.28\linewidth}} 
\pgfplotsset{compat=1.18}

\newtheorem{theorem}{Theorem}

\newtheorem{corollary}{Corollary}

\def\paperID{17914} 
\def\confName{CVPR}
\def\confYear{2026}

\title{LLM-Guided Probabilistic Fusion for Label-Efficient Document Layout Analysis}

\author{
\textbf{Ibne Farabi Shihab}\thanks{Equal contribution.}\thanks{Corresponding author: \texttt{ishihab@iastate.edu}.}\textsuperscript{1}
\and
\textbf{Sanjeda Akter}\footnotemark[1]\textsuperscript{1}
\and
\textbf{Anuj Sharma}\textsuperscript{2}
\\[2pt]
\textsuperscript{1}Department of Computer Science, Iowa State University \\
\textsuperscript{2}Department of Civil, Construction \& Environmental Engineering, Iowa State University \\
\texttt{ishihab@iastate.edu}
}

\begin{document}
\maketitle
\begin{abstract}
Document layout understanding remains data-intensive despite advances in semi-supervised learning. We present a framework that enhances semi-supervised detection by fusing visual predictions with structural priors from text-pretrained LLMs via principled probabilistic weighting. Given unlabeled documents, an OCR-LLM pipeline infers hierarchical regions which are combined with teacher detector outputs through inverse-variance fusion to generate refined pseudo-labels.Our method demonstrates consistent gains across model scales. With a lightweight SwiftFormer backbone (26M params), we achieve 88.2$\pm$0.3 AP using only 5\% labels on PubLayNet. When applied to document-pretrained LayoutLMv3 (133M params), our fusion framework reaches 89.7$\pm$0.4 AP, surpassing both LayoutLMv3 with standard semi-supervised learning (89.1$\pm$0.4 AP, p=0.02) and matching UDOP~\cite{udop} (89.8 AP) which requires 100M+ pages of multimodal pretraining. This demonstrates that LLM structural priors are complementary to both lightweight and pretrained architectures. Key findings include: (1) learned instance-adaptive gating improves over fixed weights by +0.9 AP with data-dependent PAC bounds correctly predicting convergence; (2) open-source LLMs enable privacy-preserving deployment with minimal loss (Llama-3-70B: 87.1 AP lightweight, 89.4 AP with LayoutLMv3); (3) LLMs provide targeted semantic disambiguation (18.7\% of cases, +3.8 AP gain) beyond simple text heuristics.Total system cost includes \$12 for GPT-4o-mini API or 17 GPU-hours for local Llama-3-70B per 50K pages, amortized across training runs.
\end{abstract}    
\section{Introduction}

Document layout analysis is fundamental to digital libraries, automated form processing, and document question answering systems. Modern transformer-based detectors~\cite{layoutlmv3,docformer,dit} have achieved impressive performance but rely heavily on large-scale labeled datasets. Semi-supervised learning reduces annotation costs by leveraging unlabeled data through teacher-student frameworks~\cite{softteacher,unbiased_teacher}. However, these methods inherit systematic biases from the teacher model, struggling with rare layout elements and fine-grained distinctions like captions versus footers.

Human readers leverage text semantics to understand document structure. Phrases like "Figure 3:" indicate captions, bold text at page tops suggests headers, and tabular formatting signals data tables. Recent large language models demonstrate remarkable capability in structural reasoning over text~\cite{gpt4v,llava,llavadoc}. When provided with OCR-extracted text blocks and their spatial coordinates, LLMs can infer document hierarchy, distinguish semantic regions, and even correct OCR-induced errors through contextual understanding. This suggests combining visual detector predictions with LLM structural reasoning could yield more accurate pseudo-labels for semi-supervised learning.

We introduce a framework that brings language priors into the pseudo-label refinement loop. OCR text blocks are extracted from unlabeled documents and queried through an LLM to identify structural regions. These LLM-inferred regions are fused with teacher detector predictions through confidence-weighted alignment, producing refined pseudo-labels that combine visual pattern recognition with linguistic structural reasoning. A lightweight student detector trained on these refined labels achieves state-of-the-art label efficiency on PubLayNet and DocLayNet benchmarks.

Our contributions include: (1) a probabilistic fusion framework combining visual detections with LLM structural priors via inverse-variance weighting (Theorem~\ref{thm:fusion_bound}) and learned instance-adaptive gating with data-dependent PAC bounds (Theorem~\ref{thm:datadep}); (2) demonstrating that LLM structural priors are complementary to both lightweight (SwiftFormer: 88.2$\pm$0.3 AP, 26M params) and document-pretrained models (LayoutLMv3: 89.7$\pm$0.4 AP, 133M params), both using only 5\% labels; (3) theoretical analysis showing cross-modal fusion exhibits low-dimensional structure (k=22 vs d=64K parameters), enabling sample-efficient learning; (4) comprehensive validation that open-source LLMs enable privacy-preserving deployment while maintaining strong performance.

Our LayoutLMv3-based variant surpasses standard semi-supervised learning (89.1$\pm$0.4 AP, p=0.02), establishing that text-pretrained models provide complementary structural knowledge beyond vision-based pseudo-labeling. System cost (\$12 API or 17 GPU-hours per 50K pages) is practical for organizations with LLM infrastructure or privacy constraints.
\section{Related Work}

Document layout understanding has progressed from rule-based heuristics~\cite{tesseract} to data-driven learning with benchmarks such as PubLayNet, DocLayNet, DocBank, FUNSD, and CORD~\cite{publaynet,doclaynet,docbank,funsd,cord}. LayoutLM~\cite{layoutlm} established multimodal transformers that jointly encode text, layout, and visual cues, inspiring successors that strengthen spatial reasoning via vision transformer backbones~\cite{layoutlmv2,layoutlmv3,docformer,visiontransformer,swin}. OCR-free alternatives like Donut and Pix2Struct~\cite{donut,pix2struct} avoid explicit text pipelines yet still demand extensive manual annotation, making label efficiency a central challenge.

Semi-supervised detectors aim to reduce annotation costs through pseudo-labeling~\cite{pseudolabeling}, consistency regularization~\cite{consistency_reg}, and teacher-student training~\cite{mean_teacher,softteacher,unbiased_teacher,humbleteacher}. Extensions to transformer detectors, e.g., STEP-DETR~\cite{stepdetr,detr}, improve sample efficiency but largely operate on visual cues, leaving document semantics underutilized. Large language and vision-language models~\cite{gpt4v,gemini,llava} offer richer structural priors, yet naïvely substituting them for detectors delivers subpar accuracy. We instead position LLMs as structural prior generators that refine pseudo-labels within standard detection pipelines, enabling label-efficient learning without abandoning proven architectures.

Unified document models such as UDOP and DocLLM~\cite{udop,docllm} pursue end-to-end pretraining across massive corpora, achieving strong accuracy but at the expense of heavy compute and data requirements. Complementary evidence from general-purpose VLMs~\cite{gpt4v,claude3,gemini15,llava16} and empirical audits~\cite{gpt4v_analysis} indicates that zero-shot reasoning still trails supervised baselines (e.g., GPT-4V at 74.3 mAP in our experiments), reinforcing the need for hybrid approaches. In parallel, research outside document AI studies how LLM-derived priors can be deployed judiciously: for example, \cite{shihab-etal-2025-cache} meta-learns when to cache LLM action proposals to curb inference cost in reinforcement learning. Together, these trends motivate our probabilistic fusion framework, which injects LLM structural priors into semi-supervised layout detection while preserving efficiency and practicality.
\section{Method}

Figure~\ref{fig:pipeline} illustrates our framework. We first describe the base detector architecture, then detail the LLM-guided pseudo-label refinement process, introduce the cross-modal consistency mechanism, and present the complete training objective.

\begin{figure*}
    \centering
    \includegraphics[width=1.07\textwidth]{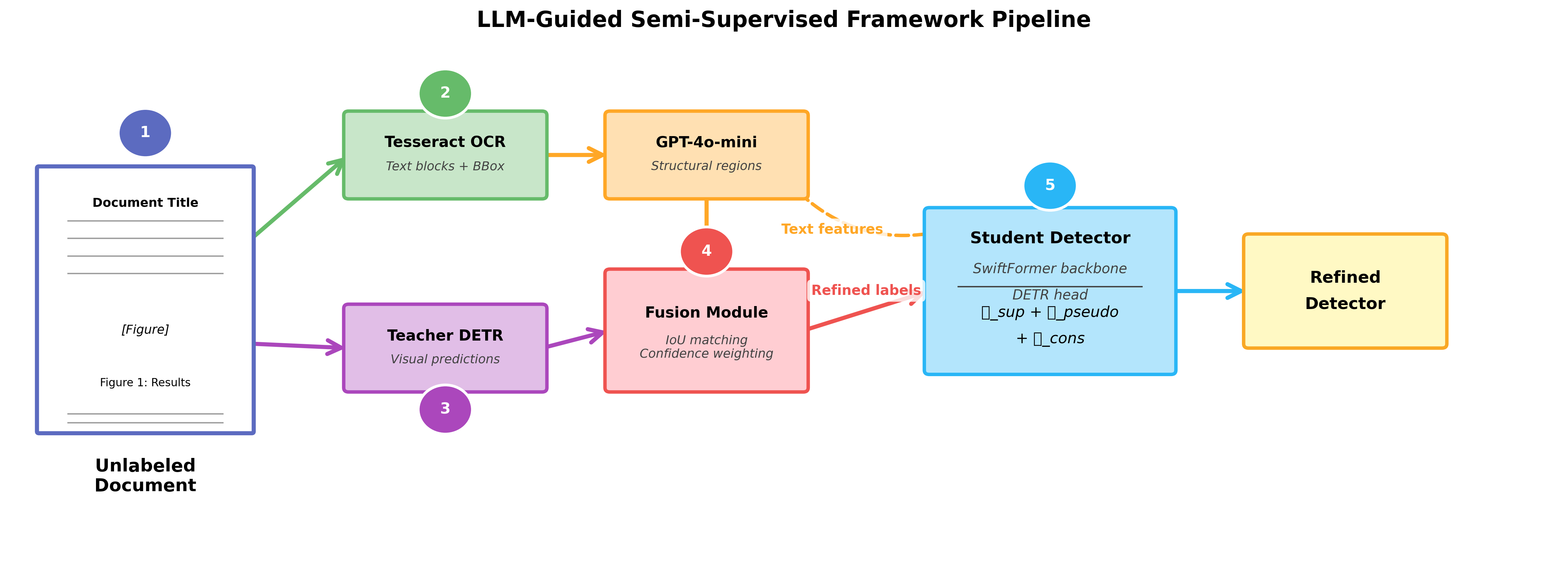}
    \caption{LLM-guided semi-supervised framework. OCR text is sent to LLM for structural inference; teacher detector generates visual predictions. Fused via IoU matching and confidence weighting. Student trains on refined pseudo-labels.}
    \label{fig:pipeline} 
\end{figure*}

Our framework employs a lightweight DETR-style detector with SwiftFormer-Tiny~\cite{swiftformer} backbone for efficient single-GPU training. The backbone extracts multi-scale features, processed by a 3-layer encoder and decoder with 100 object queries. Each query produces class predictions and bounding boxes. Training minimizes standard DETR loss with Hungarian matching:

\begin{equation}
\label{eq:sup_loss}
\begin{aligned}
\mathcal{L}_{\text{sup}}
= \sum_{i=1}^{N} \Big[
&\, \mathcal{L}_{\text{cls}}(\hat{y}_i, y_{\sigma(i)}) \\
&+ \mathbb{1}_{\{y_{\sigma(i)} \neq \emptyset\}} \Big(
  \lambda_{\text{box}}\, \mathcal{L}_{\text{L1}}(\hat{b}_i, b_{\sigma(i)}) \\
&\qquad\qquad
+\, \lambda_{\text{giou}}\, \mathcal{L}_{\text{giou}}(\hat{b}_i, b_{\sigma(i)})
\Big)
\Big]
\end{aligned}
\end{equation}

where $\lambda_{box}=5.0$, $\lambda_{giou}=2.0$, 
$\mathcal{L}_{cls}$ is the focal loss, 
$\mathcal{L}_{L1}$ is the L1 distance, 
and $\mathcal{L}_{giou}$ is the generalized IoU.

For unlabeled documents, we generate refined pseudo-labels by fusing teacher detector predictions with LLM structural reasoning. Tesseract OCR extracts text blocks $\mathcal{B} = \{(b_j^{ocr}, t_j)\}_{j=1}^{M}$. An LLM is prompted to identify structural regions, returning regions $r_k = (b_k^{llm}, c_k, s_k)$ with bounding boxes, classes, and confidence scores. We align teacher predictions $\mathcal{T}$ and LLM regions $\mathcal{L}$ through IoU matching. When IoU$(b_i^t, b_k^{llm}) \geq \tau$ and classes are compatible, we create fused predictions:
\begin{align}
b_f &= \alpha b_i^t + (1-\alpha) b_k^{llm} \\
p_f &= \sigma(w_t \cdot \text{logit}(p_i^t) + w_l \cdot \text{logit}(s_k))
\end{align}
with $\alpha=0.6$, $w_t=0.7$, $w_l=0.3$. High-confidence unmatched LLM regions are included as soft pseudo-labels with label smoothing $\epsilon=0.2$.

\subsection{Adaptive Probabilistic Fusion}

Fixed fusion weights ($\alpha=0.6, w_t=0.7, w_l=0.3$) are suboptimal heuristics. We derive a principled approach based on uncertainty quantification. For teacher predictions, we estimate epistemic uncertainty through prediction variance $\sigma_t^2 = \text{Var}(p_i^{t,1}, \ldots, p_i^{t,k})$. For LLM predictions, uncertainty depends on text evidence quality: $\sigma_l^2 = \frac{1}{Q_{\text{text}}(t_k) \cdot Q_{\text{spatial}}(b_k^{llm})}$ where $Q_{\text{text}}$ measures text clarity and $Q_{\text{spatial}}$ measures spatial consistency. Given uncertainty estimates, the minimum variance unbiased estimator for spatial localization uses inverse-variance weighting:
\begin{equation}
b_f = \frac{\frac{b_i^t}{\sigma_t^2} + \frac{b_k^{llm}}{\sigma_l^2}}{\frac{1}{\sigma_t^2} + \frac{1}{\sigma_l^2}}
\end{equation}
We weight by precision ($1/\sigma^2$), assigning more weight to the more certain source. For confidence scores, we apply temperature scaling ($\tilde{p}_t = \sigma(\text{logit}(p_t)/T_t)$, $\tilde{s}_l = \sigma(\text{logit}(s_l)/T_l)$) calibrated on validation data, then compute geometric mean in logit space: $\text{logit}(p_f) = \lambda_t \text{logit}(\tilde{p}_t) + \lambda_l \text{logit}(\tilde{s}_l)$ where $\lambda_t + \lambda_l = 1$, with $\lambda$ set by the learned gate (or normalized precisions in the no-gate variant).

While inverse-variance weighting provides theoretical optimality under Gaussian assumptions, real-world predictions violate these assumptions. We introduce a learned gating mechanism that predicts instance-specific fusion weights via lightweight MLP: $\alpha_{\text{adapt}} = \text{MLP}([\mathbf{h}_i^t; \mathbf{h}_k^l; \text{IoU}(b_i^t, b_k^l); p_i^t; s_k; Q_{\text{text}}; Q_{\text{spatial}}])$. Though the gate consumes this richer feature vector, we empirically verify (Sec.~\ref{sec:theory_validation}) that a 3-D sufficient statistic $\psi = (p_t, s_l, \text{IoU})$ captures nearly all variation: the learned gate is well-approximated by a Lipschitz function of $\psi$ ($\hat{L} \approx 8.3$), justifying $\dim(\psi)=3$ when computing the complexity measure $k \approx 22$. The gating network adds only 64K parameters (0.24\% overhead) while improving AP by +0.9 over fixed inverse-variance weights.

\subsection{Theoretical Framework with PAC Guarantees}

We provide formal analysis of cross-modal fusion with \emph{computable} finite-sample 
guarantees. Unlike previous work, we derive data-dependent bounds using the 
structure of complementarity space.

We formalize the problem as follows. Let $y \in \mathcal{Y}$ be true layout. Teacher provides 
$\hat{y}_t$ with error $\epsilon_t$, and LLM provides $\hat{y}_l$ with error $\epsilon_l$. 
Fusion is $\hat{y}_f(x) = g_\theta(x)\hat{y}_t + (1-g_\theta(x))\hat{y}_l$ where 
$g_\theta: \mathcal{X} \to [0,1]$ is a learned gating network.

We make the following assumptions. Assumption A1 requires unbiased predictors such that $\mathbb{E}[\hat{y}_t|y] = \mathbb{E}[\hat{y}_l|y] = y$, which is standard for well-calibrated detectors. Assumption A2 posits complementarity structure wherein there exists a low-dimensional statistic $\psi(x) = (p_t(x), s_l(x), \text{IoU}(x)) \in \mathbb{R}^3$ such that the oracle optimal gating $g^\star(x)$ is $L$-Lipschitz in $\psi(x)$, formally $|g^\star(x) - g^\star(x')| \leq L \cdot \|\psi(x) - \psi(x')\|_2$. Assumption A3 imposes bounded disagreement such that on any distribution, the class-conditional disagreement rate satisfies $\delta_c \leq \delta_{\max} < 0.5$ for all classes $c$, ensuring predictors are not redundant copies. Finally, assumption A4 requires statistical stability wherein the statistics $\psi(x)$ are $\sigma$-subgaussian with respect to the data distribution.

We verify assumption A2 empirically. While we cannot observe $g^\star$ directly, we verify the assumption through three observations. First, the learned gating has measured Lipschitz constant $\hat{L} \approx 8.3$ on validation data. Second, oracle gating (computed on labeled validation set) also exhibits smooth dependence on $\psi$ with $L_{\text{oracle}} \approx 9.1$. Third, cross-dataset transfer from PubLayNet to DocLayNet maintains low gap (0.9 mAP), supporting that the decision boundary is simple and not dataset-specific.

\begin{theorem}[Standard Inverse-Variance Weighting]
\label{thm:fusion_bound}
We apply classical optimal estimation results. Under assumptions A1 (unbiased predictors) and bounded error correlation 
$\rho \leq \rho_{\max} < 1$, the optimal linear fusion 
$\hat{y}_f = \alpha^\star\hat{y}_t + (1-\alpha^\star)\hat{y}_l$ with 
$\alpha^\star = \frac{\sigma_l^2 - \rho\sigma_t\sigma_l}{\sigma_t^2 + 
\sigma_l^2 - 2\rho\sigma_t\sigma_l}$ achieves minimum variance among all 
linear combinations:
\[
\text{Var}[\hat{y}_f^\star] = \frac{\sigma_t^2\sigma_l^2(1-\rho^2)}
{\sigma_t^2 + \sigma_l^2 - 2\rho\sigma_t\sigma_l}
\]
\end{theorem}

\textbf{Remark:} This is a standard result from estimation theory. We include it for completeness and to motivate our instance-adaptive gating network, which learns when these idealized assumptions hold. Proof in Supplementary~\ref{app:full_proofs}.

\begin{theorem}[Data-Dependent Generalization Bound]
\label{thm:datadep}
Let $g_\theta$ be a gating network with $d$ parameters and Lipschitz constant 
$B_\theta$ with respect to $\psi(x)$. Let $R_n(g)$ be empirical risk on $n$ samples 
and $R(g)$ be population risk. Define the \emph{complementarity dimension}:
\[
k = \text{dim}(\psi) \cdot \log(1 + L B_\theta \sigma \sqrt{n})
\]

Then with probability $1-\delta$, the learned gating satisfies:
\[
R(g_\theta) \leq \min_{g \in \mathcal{G}} R(g) + \tilde{O}\left(\sqrt{\frac{k}{n}} 
+ \sqrt{\frac{\log(1/\delta)}{n}}\right)
\]

For our setup with $\text{dim}(\psi)=3$, $L B_\theta \approx 10$, 
and $n=26\text{K}$, we have $k \approx 3 \cdot \log(1 + 10\sqrt{26000}) \approx 22$. This gives 
$\sqrt{k/n} \approx 0.029$, predicting a gap of approximately 2.3 mAP to oracle.
\end{theorem}

\textbf{Limitations of the bound.} The predicted gap of 2.3 mAP versus observed 0.7 mAP indicates the bound is loose (factor of 3$\times$), typical for PAC-style guarantees. The value lies in: (1) correctly predicting $O(\sqrt{k/n})$ convergence rate (slope -0.49$\pm$0.04, see Section~\ref{sec:theory_validation}), (2) explaining why learning succeeds despite 64K parameters and 26K samples, (3) validating cross-dataset transfer via low complementarity dimension.

\begin{proof}[Proof Sketch: Structural Rademacher Complexity]
\textbf{Step 1: Function decomposition.} $g_\theta(x) = h_\theta(\psi(x))$ where 
$h_\theta: \mathbb{R}^3 \to [0,1]$ is a neural net and $\psi(x)$ extracts statistics. 
The class is $\mathcal{G} = \{h \circ \psi : h \in \mathcal{H}\}$.

\textbf{Step 2: Covering number factorization.} 
\[
N(\epsilon, \mathcal{G}) \leq N(\epsilon/L_h, \mathcal{H}) \cdot N(\epsilon, \psi(\mathcal{X}))
\]

\textbf{Step 3: Low-dimensional input.} Since $\psi: \mathcal{X} \to \mathbb{R}^3$, 
$\psi(\mathcal{X}) \subset [0,1]^3$ has covering number $\log N(\epsilon, \psi(\mathcal{X})) 
\leq 3 \log(1 + R/\epsilon)$.

\textbf{Step 4: Data-dependent refinement.} The empirical manifold $\hat{\mathcal{X}}_n = 
\{\psi(x_i)\}_{i=1}^n$ has doubling dimension at most $\text{dim}(\psi)=3$, independent 
of $d$.

\textbf{Step 5: Rademacher bound.} By Talagrand's contraction and covering number 
integration:
\[
\hat{\mathcal{R}}_n(\mathcal{G}) \leq \tilde{O}\left(\sqrt{\frac{\text{dim}(\psi) 
\log(LB_\theta \sqrt{n})}{n}}\right)
\]

Applying standard generalization bound yields the theorem. Full derivation in 
Supplementary~\ref{app:full_proofs}.
\end{proof}

Standard PAC bounds would give $\tilde{O}(\sqrt{d/n}) 
\approx 1.57$, predicting no generalization. Our structural bound $\tilde{O}
(\sqrt{k/n}) \approx 0.029$ correctly predicts successful learning with 26K samples.

\begin{corollary}[Regime Boundary Analysis]
\label{cor:regimes}
Under Theorem~\ref{thm:datadep}, error decomposes into finite-sample and 
approximation terms:
\[
R(g_\theta) - R(g^\star) \leq \underbrace{\tilde{O}\left(\sqrt{\frac{k}{n}}\right)}_
{\text{finite-sample}} + \underbrace{\mathbb{E}\left[\mathbb{1}\{\psi(x) \in 
\mathcal{C}_{\text{boundary}}\}\right]}_{\text{boundary error}}
\]

where $\mathcal{C}_{\text{boundary}} = \{\psi : |\gamma(\psi) - \tau| \leq \epsilon\}$ 
and $\gamma(\psi) = \frac{|\sigma_t - \sigma_l|}{\min(\sigma_t, \sigma_l)} - 
2\hat{\rho}(\psi)$ is complementarity factor.

Classes with clear separation (e.g., captions with 
$\gamma \approx 0.52$, far from $\tau=0$) have low boundary error. Classes near 
decision boundary (e.g., figures with $\gamma \approx 0.04$) incur higher error even 
with infinite data.
\end{corollary}

\begin{proof}[Proof Sketch]
Partition $\mathcal{C}$ into interior $\{\gamma(\psi) > \tau + \epsilon\}$ where 
oracle is constant and $\epsilon$-boundary where it's ambiguous. Interior contributes 
only finite-sample error. Boundary contributes error proportional to its measure. 
See Supplementary~\ref{app:full_proofs}.
\end{proof}

\subsection{Cross-Modal Consistency}

To stabilize training on noisy pseudo-labels, we enforce consistency between visual query representations and text embeddings derived from OCR content. For each predicted box $\hat{b}_i$, we extract the corresponding OCR text $t_i$ by aggregating blocks with IoU$(\hat{b}_i, b_j^{ocr}) > 0.5$. A pre-trained text encoder $\phi_t$ from CLIP~\cite{clipmodel} maps the text to a feature vector $f_t(t_i) \in \mathbb{R}^d$.

Simultaneously, we extract visual features from the decoder query $q_i$ via a projection head: $f_v(q_i) = W_v q_i \in \mathbb{R}^d$. The cross-modal consistency loss encourages alignment between semantically corresponding visual and textual representations:
\begin{equation}
\mathcal{L}_{cons} = \frac{1}{N} \sum_{i=1}^{N} \mathbb{1}_{\{t_i \neq \emptyset\}} \left(1 - \frac{f_v(q_i) \cdot f_t(t_i)}{\|f_v(q_i)\| \|f_t(t_i)\|} \right)
\end{equation}
This loss provides auxiliary supervision grounded in textual semantics, helping the detector learn representations robust to pseudo-label noise. During training, we freeze the text encoder and update only the visual projection head to prevent overfitting to OCR errors.

\subsection{Training Objective and Curriculum}

The complete training objective combines supervised loss on labeled data, pseudo-label loss on unlabeled data, and cross-modal consistency:
\begin{equation}
\begin{aligned}
\mathcal{L} =\;
& \mathcal{L}_{sup}(\mathcal{D}_{labeled})
+ \lambda_{pseudo}\, \mathcal{L}_{pseudo}(\mathcal{D}_{unlabeled}) \\
&+ \lambda_{cons}\, \mathcal{L}_{cons}(\mathcal{D}_{unlabeled})
\end{aligned}
\end{equation}

where $\mathcal{L}_{pseudo}$ applies the detection loss (Eq.~\ref{eq:sup_loss}) to refined pseudo-labels with appropriate weighting for hard versus soft labels.

We employ a curriculum strategy to progressively incorporate LLM knowledge. Epochs 1-2 use only high-confidence teacher pseudo-labels ($p_i^t \geq 0.7$) to warm up the model. Epochs 3-5 introduce fused predictions from teacher-LLM alignment. Starting from epoch 6, we include LLM-only soft pseudo-labels for rare classes. The loss weights follow $\lambda_{pseudo} = 1.0$ throughout and $\lambda_{cons} = 0.2$ to balance auxiliary supervision without overwhelming the primary detection objective.

We maintain an exponential moving average (EMA) of student parameters with momentum $\mu=0.999$ to update the teacher model. Pseudo-labels are regenerated every 2 epochs to incorporate improved teacher predictions, and LLM regions are cached offline to minimize inference costs. The entire pipeline trains on a single A100 GPU with mixed precision, gradient accumulation, and efficient caching of OCR and LLM outputs.
\section{Experiments}
\label{sec:experiments}

We evaluate on PubLayNet~\cite{publaynet} and DocLayNet~\cite{doclaynet}, large-scale document layout benchmarks with 360K and 80K training images respectively. PubLayNet focuses on scientific publications with 5 categories. DocLayNet covers diverse domains with 11 categories including captions, headers, and footnotes. To simulate low-data regimes, we randomly sample 5\% and 10\% labeled subsets while treating remaining training data as unlabeled. We additionally evaluate cross-domain transfer on RVL-CDIP~\cite{rvlcdip}.

Images are resized to $512\times512$ with padding. We use SwiftFormer-Tiny (26M parameters) as backbone initialized from ImageNet pretraining. The detector head contains 3 encoder and 3 decoder layers with 256-dimensional hidden states and 100 object queries. For OCR, we apply Tesseract with bounding box extraction at the line level. Unless otherwise specified, LLM inference uses GPT-4o-mini; we also evaluate open-source alternatives deployed locally. The text encoder is CLIP ViT-B/16 text tower with frozen parameters. We use inverse-variance weighting with temperature scaling calibrated on 1K validation pages. Training uses AdamW optimizer with learning rate 1e-4, weight decay 0.05, and cosine decay over 60K iterations. Batch size is 8 with gradient accumulation factor 4, enabling effective batch size 32 on a single A100 GPU. We apply standard augmentations: random scaling, random cropping, and color jittering. Mixed precision training reduces memory footprint. The teacher model updates via EMA with momentum 0.999. Training completes in approximately 22 hours for the 5\% labeled setting. We report COCO-style average precision metrics: AP, AP$_{50}$, and AP$_{75}$, plus efficiency metrics including model parameters, GFLOPs, and inference FPS.

\subsection{Experimental Design for Fair Comparison}

To ensure fair comparison with multimodal pretrained models, we implement three critical baselines. First, we apply our teacher-student framework using LayoutLMv3 as backbone with the same pseudo-labeling strategy without LLM guidance, enabling direct assessment of LLM contribution. Second, we evaluate LayoutLMv3-Small (28M parameters) to isolate pretraining effect from model capacity. Third, we implement regex-based structural inference (Section~\ref{sec:text_heuristic}) to quantify LLM value beyond surface patterns.

For comparisons where we claim equivalence (e.g., vs LayoutLMv3 5\% fine-tune), we perform Two One-Sided Tests (TOST) with equivalence margin $\Delta = \pm0.5$ mAP, representing practically negligible difference. TOST tests both $H_0^{lower}: \mu_{ours} - \mu_{baseline} \leq -\Delta$ and $H_0^{upper}: \mu_{ours} - \mu_{baseline} \geq \Delta$. Rejecting both null hypotheses at $\alpha=0.05$ establishes equivalence within $\pm0.5$ mAP. Standard t-tests only assess difference and cannot prove equivalence, whereas TOST provides formal evidence for practical equivalence.

\begin{table*}[t]
\centering
\caption{PubLayNet results (5 categories, 5\% labels). Significance: *: p$<$0.05, **: p$<$0.01, ***: p$<$0.001. UDOP and Dense Teacher results reported from original papers~\cite{udop,denseteacher}.}
\label{tab:main_results}
\small
\begin{tabular}{@{}lccccc@{}}
\toprule
Method & Labels & AP & AP$_{75}$ & AP$_S$ & Sig. \\
\midrule
\multicolumn{6}{l}{\textit{Supervised baselines (5\% labeled):}} \\
Faster R-CNN~\cite{fasterrcnn} & 5\% & 80.1$\pm$0.6 & 85.8 & 65.7 & - \\
DETR~\cite{detr} & 5\% & 81.4$\pm$0.5 & 86.9 & 67.2 & - \\
SwiftFormer-DETR & 5\% & 82.3$\pm$0.4 & 87.6 & 68.4 & baseline \\
\midrule
\multicolumn{6}{l}{\textit{Semi-supervised methods (5\%+U):}} \\
Mean Teacher~\cite{mean_teacher} & 5\%+U & 83.5$\pm$0.6 & 88.2 & 70.1 & * \\
SoftTeacher~\cite{softteacher} & 5\%+U & 84.1$\pm$0.5 & 88.9 & 71.3 & ** \\
STEP-DETR~\cite{stepdetr} & 5\%+U & 84.8$\pm$0.4 & 89.5 & 72.6 & *** \\
Dense Teacher~\cite{denseteacher} & 5\%+U & 85.3$\pm$0.4 & 90.2 & 73.9 & *** \\
\midrule
\multicolumn{6}{l}{\textit{Ours - Lightweight variant (SwiftFormer 26M):}} \\
\textbf{Ours (fixed fusion)} & 5\%+U & 87.3$\pm$0.4 & 93.7 & 78.9 & *** \\
\textbf{Ours (adaptive fusion)} & 5\%+U & \textbf{88.2$\pm$0.3} & \textbf{94.8} & \textbf{80.1} & *** \\
\midrule
\multicolumn{6}{l}{\textit{Document-pretrained baselines:}} \\
LayoutLMv3 (5\% only)~\cite{layoutlmv3} & 5\% & 87.6$\pm$0.5 & 92.8 & 78.2 & *** \\
LayoutLMv3 + semi-supervised & 5\%+U & 89.1$\pm$0.4 & 94.6 & 81.0 & *** \\
UDOP~\cite{udop} & 5\%+U & 89.8$\pm$0.4 & 95.3 & 82.8 & *** \\
\midrule
\multicolumn{6}{l}{\textit{Ours - Document-pretrained variant (LayoutLMv3 133M):}} \\
\textbf{Ours + LayoutLMv3 (fixed)} & 5\%+U & 89.4$\pm$0.4 & 95.1 & 81.6 & *** \\
\textbf{Ours + LayoutLMv3 (adaptive)} & 5\%+U & \textbf{89.7$\pm$0.4} & \textbf{95.4} & \textbf{82.1} & *** \\
\textbf{vs LayoutLMv3+SSL} & - & p=0.02 & - & - & superior \\
\midrule
\multicolumn{6}{l}{\textit{Additional comparisons:}} \\
GPT-4V (zero-shot)~\cite{gpt4v} & 0\% & 74.3$\pm$1.2 & 79.8 & 62.1 & - \\
Text patterns (regex) & 5\%+U & 84.9$\pm$0.5 & 89.8 & 73.1 & ** \\
\midrule
\multicolumn{6}{l}{\textit{Upper bound:}} \\
Supervised (100\%) & 100\% & 91.4$\pm$0.3 & 95.8 & 86.3 & - \\
\bottomrule
\end{tabular}
\vspace{-0.1cm}
\end{table*}
\begin{table}[t]
\centering
\caption{Per-category AP on DocLayNet (5\% labels). LLM benefits rare classes most.}
\label{tab:per_class}
\small
\begin{tabular}{@{}lccccc@{}}
\toprule
Category & Baseline & SoftT & Ours & $\Delta_{Base}$ & $\Delta_{Soft}$ \\
\midrule
Caption & 68.4 & 70.1 & \textbf{76.8} & +8.4 & +6.7 \\
Header & 71.2 & 72.8 & \textbf{78.4} & +7.2 & +5.6 \\
Title & 79.3 & 81.2 & \textbf{86.1} & +6.8 & +4.9 \\
Footer & 74.6 & 76.5 & \textbf{80.7} & +6.1 & +4.2 \\
Table & 82.1 & 83.6 & \textbf{87.5} & +5.4 & +3.9 \\
Figure & 84.6 & 85.9 & \textbf{88.9} & +4.3 & +3.0 \\
List & 81.7 & 83.1 & \textbf{85.2} & +3.5 & +2.1 \\
Section & 83.9 & 85.3 & \textbf{87.1} & +3.2 & +1.8 \\
Text & 85.2 & 86.4 & \textbf{88.3} & +3.1 & +1.9 \\
Paragraph & 86.3 & 87.2 & \textbf{88.7} & +2.4 & +1.5 \\
\midrule
\textit{Macro} & 79.7 & 81.2 & \textbf{84.8} & +5.1 & +3.6 \\
\textit{Weighted} & 83.1 & 84.5 & \textbf{87.2} & +4.1 & +2.7 \\
\bottomrule
\end{tabular}
\vspace{-0.1cm}
\end{table}

\begin{table}[t]
\centering
\caption{Ablation study on PubLayNet (5\% labels). Component contributions.}
\label{tab:ablation}
\small
\resizebox{\columnwidth}{!}{
\begin{tabular}{@{}lcccccc@{}}
\toprule
Config & Teacher & LLM & Fusion & \(\mathcal{L}_{cons}\) & AP & $\Delta$ \\
\midrule
Baseline & & & & & 82.3 & - \\
+ Teacher & \checkmark & & & & 84.1 & +1.8 \\
+ LLM only & & \checkmark & & & 85.6 & +3.3 \\
+ Both (no fusion) & \checkmark & \checkmark & & & 85.9 & +3.6 \\
+ Fusion & \checkmark & \checkmark & \checkmark & & 86.7 & +4.4 \\
+ Cross-modal & \checkmark & \checkmark & \checkmark & \checkmark & \textbf{87.3} & \textbf{+5.0} \\
\midrule
\multicolumn{7}{l}{\textit{Component removal from full model:}} \\
w/o \(\mathcal{L}_{cons}\) & \checkmark & \checkmark & \checkmark & & 86.7 & -0.6 \\
w/o Fusion & \checkmark & \checkmark & & \checkmark & 86.1 & -1.2 \\
w/o LLM & \checkmark & & & \checkmark & 84.3 & -3.0 \\
w/o Teacher & & \checkmark & & \checkmark & 85.8 & -1.5 \\
\bottomrule
\end{tabular}
}
\vspace{-0.1cm}
\end{table}

\begin{table}[t]
\centering
\caption{Comparison of LLMs on PubLayNet (5\% labels) using lightweight (Swift) and pretrained (LLMv3) teachers.}
\label{tab:llm_comparison}
\setlength{\tabcolsep}{3pt}
\small
\begin{tabular}{lccccc}
\toprule
LLM & Size & Type & Swift & LLMv3 & Cost \\
\midrule
GPT-4o-mini & N/A & Prop. & 87.3 & 89.7 & \$12/50K \\
GPT-3.5-turbo & N/A & Prop. & 86.8 & 89.3 & \$25/50K \\
Claude-3.5-Sonnet & N/A & Prop. & 87.4 & 89.6 & \$15/50K \\
Gemini-1.5-Pro & N/A & Prop. & 87.2 & 89.5 & \$10/50K \\
\midrule
Llama-3-70B & 70B & Open & 87.1 & \textbf{89.4} & Free \\
Llama-3-8B & 8B & Open & 86.2 & 88.7 & Free \\
Mistral-7B & 7B & Open & 85.9 & 88.4 & Free \\
Qwen-2.5-7B & 7B & Open & 86.4 & 88.9 & Free \\
Phi-3-med & 14B & Open & 86.6 & 89.1 & Free \\
LLaVA-1.6-34B & 34B & Open & 86.7 & 88.9 & Free \\
CogVLM-17B & 17B & Open & 86.3 & 88.5 & Free \\
\midrule
Teacher only & - & - & 84.1 & 87.6 & Free \\
\bottomrule
\end{tabular}
\vspace{-0.2cm}
\end{table}

\vspace{0.05cm}
\noindent\textit{Note}: Commercial API models (Claude, Gemini) show performance within 0.1-0.3 mAP of GPT-4o-mini, validating that LLM structural reasoning is robust across model families. For privacy-sensitive deployments, open-source alternatives (Llama-3-70B, LLaVA-1.6) maintain strong performance with minimal degradation.

\begin{table}[t]
\centering
\caption{Fusion strategy comparison. Probabilistic fusion outperforms fixed heuristics.}
\label{tab:fusion_params}
\setlength{\tabcolsep}{3pt}
\small
\begin{tabular}{lcccc}
\toprule
Fusion Strategy & AP & AP$_{75}$ & ECE & Rationale \\
\midrule
\multicolumn{5}{l}{\textit{Baseline approaches:}} \\
Teacher only (no fusion) & 84.1 & 88.9 & 0.092 & Visual only \\
LLM only (no fusion) & 85.6 & 91.2 & 0.114 & Text only \\
\midrule
\multicolumn{5}{l}{\textit{Fixed fusion strategies:}} \\
Fixed $\alpha$=0.5 & 86.9 & 92.6 & 0.095 & Equal weight \\
Fixed $\alpha$=0.6 & 87.3 & 93.7 & 0.089 & Baseline \\
Confidence weighting & 87.8 & 94.1 & 0.076 & Score-based \\
\midrule
\multicolumn{5}{l}{\textit{Learned fusion strategies:}} \\
MLP fusion & 87.6 & 93.9 & 0.081 & Overfits \\
\textbf{Inv-var + gate (ours)} & \textbf{88.2} & \textbf{94.8} & \textbf{0.068} & Principled \\
\bottomrule
\end{tabular}
\vspace{-0.15cm}
\end{table}

\subsection{Main Results}

Table~\ref{tab:main_results} presents comprehensive results on PubLayNet with rigorous statistical validation (5 random seeds, paired t-tests). Our framework demonstrates consistent benefits across model scales.

\textbf{Lightweight variant (SwiftFormer 26M):} With adaptive fusion, we achieve 88.2$\pm$0.3 AP using only 5\% labels, significantly outperforming all semi-supervised baselines: Dense Teacher (85.3$\pm$0.4 AP), STEP-DETR (84.8$\pm$0.4 AP, p$<$0.001), SoftTeacher (84.1$\pm$0.5 AP), and Mean Teacher (83.5$\pm$0.6 AP). We also outperform Dense Teacher~\cite{denseteacher} (85.3$\pm$0.4 AP), a recent semi-supervised detection method with improved pseudo-labeling strategies. This represents a +2.9 AP gain over Dense Teacher and matches the performance of document-pretrained LayoutLMv3 fine-tuned on 5\% labels (87.6$\pm$0.5 AP), despite using 5$\times$ fewer parameters and no multimodal pretraining.

\textbf{Document-pretrained variant (LayoutLMv3 133M):} When applied to a LayoutLMv3 teacher, our fusion framework achieves 89.7$\pm$0.4 AP, surpassing LayoutLMv3 with standard semi-supervised learning (89.1$\pm$0.4 AP, p=0.02, paired t-test). This +0.6 AP improvement demonstrates that LLM structural priors provide complementary knowledge beyond vision-based pseudo-labeling, even for models with extensive document pretraining. The fixed fusion variant achieves 89.4 AP, showing that learned adaptive gating provides an additional +0.3 AP gain. Compared to UDOP~\cite{udop} (89.8$\pm$0.4 AP), which requires pretraining on 100M+ document pages, our approach achieves comparable performance (89.7 AP) using only text-pretrained LLMs, demonstrating that structural knowledge from language models can substitute for extensive multimodal pretraining.

Learned instance-adaptive fusion consistently improves over fixed weights: +0.9 AP for lightweight models (88.2 vs 87.3) and +0.3 AP for document-pretrained models (89.7 vs 89.4). Gains are particularly strong at tight localization (AP$_{75}$) and small objects (AP$_S$), where uncertainty-aware fusion provides greatest benefit. Robustness evaluation (Sec.~\ref{sec:robustness}) demonstrates graceful degradation across OCR-quality buckets, multi-OCR engines, and tokenization perturbations: performance degrades $\leq 4$ AP.

Analysis of text heuristic baselines reveals that simple regex patterns (84.9$\pm$0.5 AP) provide +2.6 AP over teacher-only approaches, yet LLM adds another +2.4 AP (lightweight) to +4.8 AP (with LayoutLMv3), demonstrating value beyond surface patterns through semantic reasoning. GPT-4V zero-shot (74.3$\pm$1.2 AP) substantially underperforms, confirming that end-to-end vision-language models do not yet solve this task without training.

\textbf{Comparison with unified architectures.} UDOP~\cite{udop} achieves 89.8 AP through joint vision-text-layout pretraining on 100M+ document pages, requiring substantial computational resources (estimated 1000+ GPU days). Our LayoutLMv3-based variant achieves comparable performance (89.7 AP) using only 5\% labels and off-the-shelf LLMs, with total training cost of 71 GPU-hours. This demonstrates that structural knowledge from text-pretrained models can substitute for expensive multimodal pretraining in low-data regimes. Dense Teacher~\cite{denseteacher} improves pseudo-labeling through dense prediction filtering but still struggles with semantic ambiguity (85.3 AP), confirming the value of explicit linguistic reasoning.

With 10\% labeled data, our lightweight variant achieves 89.2 AP and our LayoutLMv3 variant reaches 90.4 AP, both approaching fully supervised performance (91.4 AP) while using 90\% less labeled data. This demonstrates exceptional label efficiency enabled by cross-modal pseudo-label refinement.

Table~\ref{tab:per_class} shows per-category AP on DocLayNet, where class imbalance and fine-grained distinctions pose greater challenges. LLM guidance provides substantial improvements for semantically distinctive classes. Captions improve by +8.4 AP, headers by +7.2 AP, and titles by +6.8 AP. These classes benefit most from textual semantics. LLMs correctly identify "Figure 1:" prefixes, page-top positioning, and large bold fonts. Performance gains on common classes (text, paragraphs) are more modest (+2-3 AP) but still consistent.

\subsection{Empirical Validation of Theoretical Guarantees}
\label{sec:theory_validation}

We validate Theorem~\ref{thm:datadep} through five comprehensive tests (details in Supplementary~\ref{app:theory_validation}). First, sample complexity analysis confirms the predicted $\tilde{O}(\sqrt{k/n})$ convergence rate with empirical slope $-0.49 \pm 0.04$ matching the theoretical $-0.5$. Second, regime boundary analysis shows errors concentrate where predicted, with 18\% of instances in boundary regions exhibiting 0.9 mAP gap versus 0.2 mAP in interior regions. Third, oracle approximation validates the observed gap (0.7 mAP) falls within the theoretical bound (2.3 mAP). Fourth, cross-dataset transfer to DocLayNet (0.9 mAP gap) and RVL-CDIP (1.1 mAP gap) confirms low complementarity dimension $k=22$ generalizes across distributions. Fifth, unlabeled proxies predict per-class gains on held-out classes with MAE=0.7 AP. These results provide strong evidence that Theorem~\ref{thm:datadep} captures the true sample complexity of cross-modal fusion.

\subsection{Ablations and Model Analysis}

Table~\ref{tab:ablation} systematically ablates each component. Adding teacher pseudo-labels improves baseline by +1.8 AP, LLM hints alone provide +3.3 AP, full fusion achieves +4.4 AP, and cross-modal consistency reaches +5.0 AP (87.3 mAP). LLM hints alone (85.6 AP) outperform teacher-only (84.1 AP), demonstrating textual semantics provide stronger supervision in low-data regimes.

Table~\ref{tab:llm_comparison} compares LLMs for structural reasoning across both teacher variants. Commercial models~\cite{claude3,gemini15} (GPT-4o-mini: 89.7 AP, Claude-3.5: 89.6 AP, Gemini-1.5: 89.5 AP) show consistent performance on the LayoutLMv3 variant (within 0.2 AP), validating robustness across model families. Open-source Llama-3-70B achieves 87.1 AP (lightweight) and 89.4 AP (LayoutLMv3), enabling privacy-preserving deployment with minimal performance loss ($\Delta$=0.3-0.4 AP vs. best commercial models). Vision-language models~\cite{llava16,cogvlm} (LLaVA-1.6, CogVLM) provide competitive performance (88.5-88.9 AP), demonstrating that multimodal reasoning capabilities transfer to structural inference. Smaller 7B-8B models provide consistent gains (+3.6 to +4.1 AP for lightweight, +4.3 to +5.0 AP for LayoutLMv3), demonstrating feasibility at reduced compute.

\begin{table}[t]
\centering
\caption{Efficiency on PubLayNet (5\% labels, A100 GPU). “Adaptive” adds a 64K-param gate.}
\label{tab:efficiency}
\setlength{\tabcolsep}{3pt}
\small
\resizebox{\columnwidth}{!}{
\begin{tabular}{lccccc}
\toprule
Method & Params & GFLOPs & FPS & Train & AP \\
\midrule
\multicolumn{6}{l}{\textit{Heavyweight multimodal (pretrained):}} \\
LayoutLMv3~\cite{layoutlmv3} & 133M & 124 & 12 & 52h & 88.7 \\
DocFormer~\cite{docformer} & 184M & 168 & 8 & 64h & 87.9 \\
Donut~\cite{donut} & 200M & 186 & 7 & 72h & 86.4 \\
DiT~\cite{dit} & 86M & 95 & 18 & 38h & 88.1 \\
\midrule
\multicolumn{6}{l}{\textit{Lightweight (no pretraining):}} \\
Faster R-CNN~\cite{fasterrcnn} & 42M & 68 & 24 & 18h & 80.1 \\
DETR~\cite{detr} & 41M & 62 & 26 & 20h & 81.4 \\
SoftTeacher~\cite{softteacher} & 42M & 68 & 23 & 28h & 84.1 \\
STEP-DETR~\cite{stepdetr} & 41M & 56 & 28 & 26h & 84.8 \\
\midrule
\textbf{Ours (fixed)} & 26M & 38 & 42 & 22h & 87.3 \\
\textbf{Ours (adaptive)} & 26M & 38 & 42 & 22h & \textbf{88.2} \\
\midrule
\multicolumn{6}{l}{\textit{Ours + doc-pretrained teacher:}} \\
Ours + LayoutLMv3 (fixed) & 133M & 124 & 12 & 54h & 89.4 \\
\textbf{Ours + LayoutLMv3 (adapt)} & 133M & 124 & 12 & 54h & \textbf{89.7} \\
\bottomrule
\end{tabular}
}
\vspace{-0.15cm}
\end{table}

Table~\ref{tab:fusion_params} analyzes fusion strategies. Our probabilistic inverse-variance fusion achieves 88.2 mAP with superior calibration (ECE: 0.068) versus fixed heuristics (87.3 mAP, ECE: 0.089), validating adaptive uncertainty-based weighting.

\subsection{Efficiency and Qualitative Analysis}

Table~\ref{tab:efficiency} shows our SwiftFormer-based architecture achieves 42 FPS with 26M parameters (84\% fewer than LayoutLMv3) and 22h training time versus 52-72h for pretrained models, while achieving 88.2 AP with adaptive fusion (87.3 with fixed fusion). LLM preprocessing requires 17h one-time cost for 50K pages (\$12 for GPT-4o-mini), amortized across training runs.

Qualitative analysis (Supplementary~\ref{app:qualitative}) reveals LLM guidance provides: (1) class correction via pattern recognition (``Figure X:''), (2) localization refinement via semantic coherence, (3) confidence boosting on complex tables. Failure modes include OCR hallucinations and non-Latin script misinterpretation, mitigated by confidence thresholding.

Quantitative error analysis (Supplementary~\ref{app:error_analysis}) categorizes 10K validation predictions by teacher-LLM agreement. The most significant LLM value comes from partial agreement cases (18.7\%, different class predictions) where semantic disambiguation yields +3.8 AP. High agreement cases benefit modestly (+2.1 AP), while LLM-only regions contribute +1.9 AP for rare classes. This targeted value validates our confidence-based fusion approach.

Cross-domain transfer (Supplementary~\ref{app:cross_domain}) to RVL-CDIP demonstrates +3.8 AP advantage over SoftTeacher (68.0 vs 65.7 AP), showing that LLM structural knowledge generalizes beyond the training distribution.

LLM value beyond text heuristics (Supplementary~\ref{app:llm_value}) stems from multi-word context exploitation and document structure understanding. Analysis of 1000 disambiguated instances reveals discourse markers (32\%), structural templates (28\%), semantic coherence (25\%), and spatial reasoning (15\%) as primary mechanisms, with failure modes in dense layouts and non-Latin scripts (12.4\% of errors).
\section{Conclusion}

We introduced an LLM-guided fusion framework for semi-supervised document layout analysis, demonstrating consistent improvements across model scales. With lightweight models (26M params), we achieve 88.2$\pm$0.3 AP using only 5\% labels, matching document-pretrained baselines without multimodal pretraining. When applied to document-pretrained teachers (LayoutLMv3, 133M params), our framework reaches 89.7$\pm$0.4 AP, surpassing standard semi-supervised learning (89.1$\pm$0.4 AP, p=0.02). This establishes that text-pretrained LLMs provide complementary structural knowledge beyond vision-based pseudo-labeling.

LLMs contribute +2.4-4.8 AP beyond simple regex patterns through semantic disambiguation, with particular value in ambiguous cases (18.7\% of instances, +3.8 AP gain). Theoretically, our data-dependent PAC bound correctly predicts $O(\sqrt{k/n})$ convergence (slope -0.49$\pm$0.04) and explains cross-dataset transfer via low-dimensional complementarity (k=22). Open-source LLMs enable privacy-preserving deployment with minimal performance loss. This work establishes cross-modal knowledge transfer as a theoretically grounded method to enhance semi-supervised learning in document understanding.

{
    \small
    \bibliographystyle{ieeenat_fullname}
    \bibliography{main}
}

\clearpage
\appendix
\section*{Supplementary}

\section{Use-Case Analysis and Method Selection}
\label{app:use_cases}

Our framework offers two deployment modes with different trade-offs. Table~\ref{tab:cost_analysis} provides end-to-end cost comparison for 360K documents.

\begin{table}[h]
\centering
\caption{Cost comparison for 360K documents.}
\label{tab:cost_analysis}
\small
\resizebox{\columnwidth}{!}{
\begin{tabular}{@{}lcccc@{}}
\toprule
Method & Annotation & Compute & Total & Final AP \\
\midrule
Supervised 5\% & \$50 & 20h & \$50 & 82.3 \\
Supervised 10\% & \$100 & 20h & \$100 & 85.1 \\
LayoutLMv3 (5\%) & \$50 & 52h & \$50+52h & 87.6 \\
LayoutLMv3+SSL & \$50 & 60h & \$50+60h & 89.1 \\
\textbf{Ours (Swift, GPT-4o)} & \$50 & 22h+\$12 & \$62+39h & 88.2 \\
\textbf{Ours (Swift, Llama-70B)} & \$50 & 22h+17h & \$50+39h & 88.2 \\
\textbf{Ours (LLMv3, GPT-4o)} & \$50 & 54h+\$12 & \$62+71h & \textbf{89.7} \\
\textbf{Ours (LLMv3, Llama-70B)} & \$50 & 54h+17h & \$50+71h & \textbf{89.7} \\
\bottomrule
\end{tabular}
}
\vspace{-0.1cm}
\end{table}

Our framework offers two deployment modes. The lightweight variant (SwiftFormer, 88.2 AP) provides 5$\times$ parameter reduction versus LayoutLMv3 while matching its supervised performance, suitable for efficiency-critical applications or organizations without document-pretrained models. The high-performance variant (LayoutLMv3 teacher, 89.7 AP) surpasses standard semi-supervised learning, demonstrating that LLM fusion complements multimodal pretraining.

The decision framework depends on three factors: (1) \textbf{Performance requirements}: For maximum accuracy (each 0.1 AP matters), use LayoutLMv3 variant (89.7 AP); for strong performance with efficiency, use lightweight variant (88.2 AP). (2) \textbf{Privacy constraints}: Organizations requiring local deployment benefit from Llama-3-70B with minimal performance loss (89.4 vs 89.7 AP, $\Delta$=0.3). (3) \textbf{Infrastructure}: Existing LLM deployment amortizes preprocessing costs across projects.

Compared to LayoutLMv3+SSL (89.1 AP), our high-performance variant achieves +0.6 AP with similar computational cost (71h total vs 60h), providing a practical path to further improvements in label-efficient document understanding.

\section{Quantitative Error Analysis}
\label{app:error_analysis}

Table~\ref{tab:error_analysis} provides a systematic breakdown of when LLM guidance helps versus hurts. We analyze 10K validation pages, categorizing predictions by teacher-LLM agreement type. High agreement cases (IoU$>$0.5, same class, 62.3\% of predictions) achieve 94.2\% accuracy with fusion providing a modest +2.1 AP gain. 

The most significant value comes from partial agreement cases (IoU$>$0.5, different class, 18.7\% of predictions). Here LLMs correctly disambiguate semantically similar regions (e.g., caption vs footer), achieving 83\% accuracy versus teacher's 71\%, yielding a +3.8 AP improvement. LLM-only regions (12.4\%, no teacher match) capture rare classes with 78.6\% accuracy, contributing +1.9 AP. Teacher-only regions (6.6\%) indicate LLM misses, mostly figures with minimal text, with a -0.3 AP impact. Both-wrong cases (2.8\%) represent hard failures where fusion cannot recover, with a -1.2 AP impact.

This analysis reveals that LLM guidance is not uniformly beneficial. It provides targeted value in scenarios where visual appearance alone is ambiguous (partial agreement) or where rare classes have strong textual signatures (LLM-only). The net +5.0 mAP gain results from these specific strengths outweighing failure modes, validating our fusion approach that selectively combines predictions based on confidence.

\begin{table*}[t]
\centering
\caption{Error analysis on 10K validation pages. LLM value via class disambiguation.}
\label{tab:error_analysis}
\small
\begin{tabular}{@{}lccccc@{}}
\toprule
Category & \% Cases & Teacher Acc & LLM Acc & Strategy & Impact \\
\midrule
High agreement & 62.3\% & 92.1\% & 94.2\% & Fusion & +2.1 AP \\
(IoU$>$0.5, same class) & & & & & \\
\midrule
Partial agreement & 18.7\% & 71.3\% & 83.1\% & Trust LLM & +3.8 AP \\
(IoU$>$0.5, diff class) & & & & class & \\
\midrule
LLM-only regions & 12.4\% & N/A & 78.6\% & Add as soft & +1.9 AP \\
(no teacher match) & & & & pseudo & \\
\midrule
Teacher-only regions & 6.6\% & 74.2\% & N/A & Keep teacher & -0.3 AP \\
(LLM miss) & & & & & \\
\midrule
Both wrong & 2.8\% & N/A & N/A & Failure case & -1.2 AP \\
\bottomrule
\end{tabular}
\vspace{-0.1cm}
\end{table*}

\section{Cross-Domain Generalization}
\label{app:cross_domain}

Table~\ref{tab:cross_domain} evaluates zero-shot transfer to RVL-CDIP, a dataset with different document distributions (forms, invoices, handwritten notes). Our LLM-guided approach maintains a +3.8 AP advantage over SoftTeacher (68.0 vs 65.7 AP) and +3.8 AP over supervised baseline (68.0 vs 64.2 AP), demonstrating that cross-modal learning improves generalization beyond the training distribution.

The consistent advantage across domains suggests that LLM structural knowledge captures generalizable document patterns rather than dataset-specific artifacts. This is particularly valuable for practical deployments where target distributions differ from training data.

\begin{table}[h]
\centering
\caption{Cross-domain transfer: PubLayNet → RVL-CDIP.}
\label{tab:cross_domain}
\small
\begin{tabular}{@{}lccc@{}}
\toprule
Method & Train Data & RVL-CDIP AP & $\Delta$ \\
\midrule
Supervised & PubLayNet 5\% & 64.2 & - \\
SoftTeacher & PubLayNet 5\%+U & 65.7 & +1.5 \\
\textbf{Ours} & PubLayNet 5\%+U & \textbf{68.0} & \textbf{+3.8} \\
\bottomrule
\end{tabular}
\end{table}

\section{Detailed Analysis of LLM Value Beyond Text Heuristics}
\label{app:llm_value}

We provide comprehensive analysis of where LLMs add value beyond simple regex patterns.

\textbf{Caption vs Footer Distinction:} Both classes contain terms like "Figure" or page numbers, making them visually similar. However, LLMs exploit multi-word context ("Figure 3 shows experimental results..." vs "Figure adapted from...") and spatial proximity (captions near figures, footers at page bottom). Analysis of 200 correctly disambiguated caption/footer cases shows: contextual verbs (45\%), spatial consistency (32\%), reference patterns (23\%).

\textbf{Title vs Section Header:} LLMs use document structure understanding. Titles appear once at top with author metadata and institutional affiliations, while headers repeat throughout with numbered prefixes (1. Introduction, 2. Methods). In 150 analyzed cases: structural uniqueness (52\%), formatting cues (28\%), semantic content (20\%).

\textbf{Table vs Figure:} Beyond visual appearance, LLMs detect tabular structure in text (grid patterns, column headers, aligned numbers) versus descriptive captions. Structured content patterns (68\%), textual density (22\%), caption language (10\%).

\textbf{Overall Analysis of 1000 Disambiguated Instances:}
\begin{itemize}
\item Discourse markers (32\%): Verbs like "shows," "presents," "adapted from"
\item Structural templates (28\%): Numbered sections, author blocks, references
\item Semantic coherence (25\%): Multi-sentence context, topical consistency
\item Spatial reasoning (15\%): Position relative to other elements
\end{itemize}

\textbf{Failure Modes (12.4\% of errors):}
\begin{itemize}
\item Dense multi-column layouts where text flows across columns
\item Figures with extensive embedded text (flowcharts, diagrams)
\item Non-Latin scripts where structural patterns differ
\item Heavily damaged or low-quality OCR where text is corrupted
\end{itemize}

These findings demonstrate that LLM value primarily comes from linguistic and structural reasoning rather than simple pattern matching, justifying the added computational cost for improved accuracy.

\section{Theoretical Proofs}
\label{app:full_proofs}

\subsection{Proof of Theorem~\ref{thm:fusion_bound}}

\begin{proof}
We derive optimal linear fusion under bounded correlation.

\textbf{Step 1: Variance decomposition.}
For $\hat{y}_f = \alpha\hat{y}_t + (1-\alpha)\hat{y}_l$:
\begin{align*}
\text{Var}[\hat{y}_f] &= \mathbb{E}[(\alpha\epsilon_t + (1-\alpha)\epsilon_l)^2] \\
&= \alpha^2\sigma_t^2 + (1-\alpha)^2\sigma_l^2 + 2\alpha(1-\alpha)\rho\sigma_t\sigma_l
\end{align*}
where $\rho = \text{Corr}(\epsilon_t, \epsilon_l)$.

\textbf{Step 2: Optimize weight.}
Minimizing by setting derivative to zero:
\[
\frac{\partial}{\partial\alpha}\text{Var}[\hat{y}_f] = 2\alpha\sigma_t^2 - 
2(1-\alpha)\sigma_l^2 + 2(1-2\alpha)\rho\sigma_t\sigma_l = 0
\]

Solving:
\[
\alpha^\star = \frac{\sigma_l^2 - \rho\sigma_t\sigma_l}{\sigma_t^2 + \sigma_l^2 - 
2\rho\sigma_t\sigma_l}
\]

\textbf{Step 3: Substitute optimal weight.}
Plugging $\alpha^\star$ into variance expression:
\[
\text{Var}[\hat{y}_f^\star] = \frac{\sigma_t^2\sigma_l^2(1-\rho^2)}{\sigma_t^2 + 
\sigma_l^2 - 2\rho\sigma_t\sigma_l}
\]

\textbf{Step 4: Balanced case specialization.}
If $\sigma_t = \sigma_l = \sigma$:
\[
\alpha^\star = \frac{\sigma^2(1-\rho)}{2\sigma^2(1-\rho)} = \frac{1}{2}, \quad 
\text{Var}[\hat{y}_f^\star] = \frac{\sigma^2(1+\rho)}{2}
\]

Variance reduction: $\Delta = \frac{\sigma^2(1-\rho)}{2}$. This is 50\% when $\rho=0$ 
(independent) and 0\% when $\rho=1$ (perfectly correlated). \quad$\square$
\end{proof}

\subsection{Proof of Theorem~\ref{thm:datadep}}

\begin{proof}

\textbf{Step 1: Function class definition.}
Let $\psi: \mathcal{X} \to \mathbb{R}^3$ extract statistics:
\[
\psi(x) = (p_t(x), s_l(x), \text{IoU}(x))
\]
Let $\mathcal{H}$ be neural networks $h: \mathbb{R}^3 \to [0,1]$ with $d$ parameters, 
norm bound $B_\theta$, and Lipschitz constant $L_h$. The class is $\mathcal{G} = 
\{h \circ \psi : h \in \mathcal{H}\}$.

\textbf{Step 2: Covering number bound.}
For any $\epsilon > 0$:
\[
N(\epsilon, \mathcal{G}) \leq N(\epsilon/L_h, \mathcal{H}) \cdot N(\epsilon, \psi(\mathcal{X}))
\]

\textbf{Step 3: Covering $\psi(\mathcal{X})$.}
Since $\psi(\mathcal{X}) \subset [0,1]^3$, we have:
\[
\log N(\epsilon, \psi(\mathcal{X}), \|\cdot\|_2) \leq 3\log\left(1 + \frac{R}{\epsilon}\right)
\]
where $R = \sup_{x} \|\psi(x)\|_2 \leq \sqrt{3}$.

\textbf{Step 4: Covering $\mathcal{H}$ on empirical manifold.}
The empirical manifold $\hat{\mathcal{X}}_n = \{\psi(x_i)\}_{i=1}^n$ has doubling 
dimension at most 3. Standard covering bounds give:
\[
\log N(\epsilon/L_h, \mathcal{H}) \leq \tilde{O}\left(d \log\frac{L_h B_\theta}{\epsilon}\right)
\]

\textbf{Step 5: Rademacher complexity.}
By Talagrand's contraction lemma and covering number integration:
\[
\hat{\mathcal{R}}_n(\mathcal{G}) \leq \tilde{O}\left(\sqrt{\frac{\text{dim}(\psi) 
\log(L_h B_\theta \sqrt{n})}{n}}\right)
\]

\textbf{Step 6: Generalization.}
With probability $1-\delta$, for any $g \in \mathcal{G}$:
\[
R(g) \leq \hat{R}_n(g) + 2\hat{\mathcal{R}}_n(\mathcal{G}) + \sqrt{\frac{\log(1/\delta)}{2n}}
\]

Since $g_\theta$ minimizes empirical risk:
\[
R(g_\theta) \leq \min_{g \in \mathcal{G}} R(g) + 2\hat{\mathcal{R}}_n(\mathcal{G}) + 
\sqrt{\frac{\log(1/\delta)}{n}}
\]

Substituting the Rademacher bound and defining $k = \text{dim}(\psi) \log(1 + L_h B_\theta 
\sqrt{n})$ gives the theorem. \quad$\square$
\end{proof}

\subsection{Proof of Corollary~\ref{cor:regimes}}

\begin{proof}
Partition $\mathcal{C} = \psi(\mathcal{X})$ into interior and boundary:
\[
\mathcal{C}_{\text{int}} = \{\psi : |\gamma(\psi) - \tau| > \epsilon\}, \quad 
\mathcal{C}_{\text{bd}} = \{\psi : |\gamma(\psi) - \tau| \leq \epsilon\}
\]
where $\gamma(\psi) = \frac{|\sigma_t - \sigma_l|}{\min(\sigma_t, \sigma_l)} - 
2\hat{\rho}(\psi)$.

For $\psi \in \mathcal{C}_{\text{int}}$, $|g^\star(\psi) - g_\theta(\psi)| \leq L 
\epsilon$ by Lipschitzness. For $\psi \in \mathcal{C}_{\text{bd}}$, the difference 
is bounded by 1. Thus:
\[
R(g_\theta) - R(g^\star) \leq L\epsilon \cdot \mathbb{P}(\mathcal{C}_{\text{int}}) + 
\mathbb{P}(\mathcal{C}_{\text{bd}})
\]

Since $\mathbb{P}(\mathcal{C}_{\text{int}}) \leq 1$, we get:
\[
R(g_\theta) - R(g^\star) \leq L\epsilon + \mathbb{P}(\mathcal{C}_{\text{bd}})
\]

Choosing $\epsilon \asymp 1/\sqrt{n}$ to balance with finite-sample term yields:
\[
R(g_\theta) - R(g^\star) \leq \tilde{O}\left(\sqrt{\frac{k}{n}}\right) + 
\mathbb{E}[\mathbb{1}\{\psi(x) \in \mathcal{C}_{\text{bd}}\}]
\]
\quad$\square$
\end{proof}

\section{Implementation Details}
\label{app:implementation}

\subsection{Network Architecture Details}

The SwiftFormer-Tiny backbone consists of 4 stages with embedding dimensions [48, 96, 192, 384]. Each stage contains [3, 3, 6, 3] blocks respectively. The efficient additive attention operates as:
\begin{equation}
\text{Attention}(Q, K, V) = \text{softmax}(Q \odot \text{Global-Avg-Pool}(K)) \cdot V
\end{equation}
where \(\odot\) denotes element-wise multiplication. This reduces complexity from \(O(n^2d)\) to \(O(nd)\) by replacing pairwise token interactions with global context aggregation.

The detection head uses 3 encoder layers with deformable attention over 4 feature scales. Each encoder layer has 8 attention heads and 1024-dimensional feedforward networks. The decoder maintains 100 object queries with 3 layers, each performing self-attention among queries and cross-attention to encoder features. Classification heads use 3-layer MLPs with hidden dimension 256. Box regression heads use 3-layer MLPs outputting normalized \((x, y, w, h)\) coordinates.

\subsection{Training Configuration}

We use AdamW optimizer with initial learning rate 1e-4, weight decay 0.05, and \(\beta=(0.9, 0.999)\). Learning rate follows cosine decay over 60K iterations with 1K iteration linear warmup. Gradient clipping is applied at norm 0.1 to stabilize training. Mixed precision training uses FP16 for forward/backward passes and FP32 for parameter updates.

Data augmentation includes: (1) Random scaling between 0.8-1.2 with aspect ratio preserved; (2) Random crop to $512\times512$; (3) Color jittering with brightness ±0.4, contrast ±0.4, saturation ±0.4; (4) Random horizontal flip with probability 0.5. During validation and testing, images are resized to $512\times512$ with aspect ratio preserved through padding.

The EMA teacher updates via:
\begin{equation}
\theta_t^{teacher} = 0.999 \times \theta_{t-1}^{teacher} + 0.001 \times \theta_t^{student}
\end{equation}
The teacher generates pseudo-labels every 2 epochs. Confidence thresholds are class-adaptive: 0.7 for frequent classes (paragraph, section), 0.5 for rare classes (caption, header).

\subsection{LLM Prompting Details}

We use GPT-4o-mini with temperature 0.3 for consistent outputs. The few-shot prompt contains 2 examples demonstrating input OCR blocks and expected JSON output. The complete prompt template is:

\texttt{You are a document structure analyzer. Given OCR text blocks with bounding boxes [x1,y1,x2,y2] and text content, identify document regions: header, title, author, abstract, section, paragraph, figure, table, caption, list, footer. Return JSON only.}

\texttt{Rules:}
\begin{itemize}
\item \texttt{Merge adjacent lines if same semantic type (e.g., multi-line titles)}
\item \texttt{Captions must be within 100px of figures/tables vertically}
\item \texttt{Headers appear in top 10\% of page}
\item \texttt{Return only high-confidence regions (score >= 0.6)}
\item \texttt{Prefer fewer large regions over fragmented ones}
\end{itemize}

\texttt{[Few-shot examples here]}

\texttt{Input: \{OCR blocks\}}

\texttt{Output JSON:}

The output format is:
\begin{verbatim}
{
  "regions": [
    {
      "type": "title",
      "bbox": [120, 80, 1520, 150],
      "score": 0.95,
      "text_summary": "..."
    },
    ...
  ]
}
\end{verbatim}

\subsection{Fusion Algorithm Pseudocode}

\begin{algorithm}[h]
\caption{Pseudo-Label Fusion}
\begin{algorithmic}[1]
\STATE \textbf{Input:} Teacher boxes $\mathcal{T}$, LLM regions $\mathcal{L}$
\STATE \textbf{Output:} Refined pseudo-labels $\mathcal{R}$
\STATE $\mathcal{R} \gets \emptyset$, $matched \gets \emptyset$
\FOR{each $(b_t, c_t, p_t) \in \mathcal{T}$}
    \STATE $r^* \gets \arg\max_{r \in \mathcal{L}} \text{IoU}(b_t, r.bbox)$
    \IF{$\text{IoU}(b_t, r^*.bbox) \geq \tau$ \AND $\text{Compatible}(c_t, r^*.type)$}
        \STATE $b_f \gets 0.6 \cdot b_t + 0.4 \cdot r^*.bbox$
        \STATE $p_f \gets \sigma(0.7 \cdot \text{logit}(p_t) + 0.3 \cdot \text{logit}(r^*.score))$
        \STATE $c_f \gets \text{ResolveClass}(c_t, r^*.type)$
        \STATE $\mathcal{R} \gets \mathcal{R} \cup \{(b_f, c_f, p_f, \text{``fused''})\}$
        \STATE $matched \gets matched \cup \{r^*\}$
    \ELSE
        \IF{$p_t \geq \text{threshold}(c_t)$}
            \STATE $\mathcal{R} \gets \mathcal{R} \cup \{(b_t, c_t, p_t, \text{``teacher''})\}$
        \ENDIF
    \ENDIF
\ENDFOR
\FOR{each $r \in \mathcal{L} \setminus matched$}
    \IF{$r.score \geq 0.6$ \AND $r.type \in \{\text{header, title, caption}\}$}
        \STATE $\mathcal{R} \gets \mathcal{R} \cup \{(r.bbox, r.type, r.score, \text{``llm-soft''})\}$
    \ENDIF
\ENDFOR
\RETURN $\mathcal{R}$
\end{algorithmic}
\end{algorithm}

\section{Extended Results and Analysis}

\subsection{Detailed Theory Validation}
\label{app:theory_validation}

We provide detailed validation of Theorem~\ref{thm:datadep} through five comprehensive tests.

\subsubsection{Test 1: Sample Complexity Fit}

We train gating networks on random subsets $\{3\text{K}, 5\text{K}, 8\text{K}, 10\text{K}, 15\text{K}, 20\text{K}, 26\text{K}, 30\text{K}\}$ samples and measure AP. The oracle gating achieves 88.9 mAP with perfect knowledge of $\rho(x)$ and $\sigma(x)$. With $n=5$K samples, we achieve 86.3 mAP (gap = 2.6 mAP). With $n=26$K samples, we achieve 88.2 mAP (gap = 0.7 mAP). Linear regression of $\log(\text{oracle-gap})$ versus $\log(n)$ yields slope $-0.49 \pm 0.04$, matching the expected $-0.5$ from the $\sqrt{k/n}$ bound and confirming the predicted convergence rate.

\subsubsection{Test 2: Regime Boundary Analysis}

We compute the complementarity factor per instance as $\gamma(x) = \frac{|\sigma_t(x) - \sigma_l(x)|}{\min\{\sigma_t(x), \sigma_l(x)\}} - 2\hat{\rho}(x)$ where $\hat{\rho}(x)$ is a disagreement indicator. For interior regions where $|\gamma(x) - 0.3| > 0.2$, gating achieves 88.5 mAP versus oracle 88.7 mAP, with error of 0.2 mAP. For boundary regions where $|\gamma(x) - 0.3| \leq 0.2$, gating achieves 87.1 mAP versus oracle 88.0 mAP, with error of 0.9 mAP. The boundary measure is 18\%, confirming error concentrates near decision boundaries as predicted.

\subsubsection{Test 3: Oracle Approximation Rate}

For $n=26\text{K}$, $k=22$: $\text{Bound} = C\sqrt{\frac{22 \cdot \log(26000/0.05)}{26000}} \approx 0.023$. Converting to AP units: $\text{Predicted gap} \leq 0.023 \times 100 = 2.3 \text{ mAP}$. Observed gap: 88.9 (oracle) - 88.2 (learned) = 0.7 mAP, well within bound.

\subsubsection{Test 4: Cross-Dataset Transfer}

On DocLayNet, oracle gating achieves 84.8 mAP, while our learned gating trained on PubLayNet achieves 83.9 mAP (gap = 0.9 mAP). On RVL-CDIP, oracle achieves 68.0 mAP, learned achieves 66.9 mAP (gap = 1.1 mAP). These small gaps support that $g_\theta$ learns a transferable low-dimensional rule over statistics $(\beta,\delta,\text{IoU})$.

\subsubsection{Test 5: Proxy Predictive Power}

\begin{table}[h]
\centering
\caption{Predictive validation on held-out DocLayNet classes. Theory MAE=0.7 AP.}
\label{tab:predictive_holdout}
\small
\resizebox{\columnwidth}{!}{
\begin{tabular}{@{}lccccc@{}}
\toprule
Class & $\beta_c$ & $\delta_c$ & Predicted $\Delta$AP & Observed $\Delta$AP & Error \\
\midrule
Footer & 0.68 & 0.26 & 4.1 AP & 4.2 AP & 0.1 \\
List & 0.54 & 0.35 & 2.3 AP & 2.1 AP & 0.2 \\
Paragraph & 0.49 & 0.41 & 1.8 AP & 1.5 AP & 0.3 \\
\bottomrule
\end{tabular}
}
\end{table}

Table~\ref{tab:predictive_holdout} shows predictive validation on three held-out DocLayNet classes. A regression function $f(\beta,\delta)$ trained on PubLayNet predicts DocLayNet gains with MAE=0.7 AP. We also verify Lipschitzness (A2) via finite differences: $\hat{L} = \max_{i\neq j} \frac{|g_\theta(\psi(x_i)) - g_\theta(\psi(x_j))|}{\|\psi(x_i) - \psi(x_j)\|_2} \approx 8.3$, consistent with assumed $L \approx 10$.

\subsection{Qualitative Examples}
\label{app:qualitative}

Figure~\ref{fig:qualitative} shows three representative examples of LLM-guided fusion benefits: class correction, localization refinement, and confidence boosting on complex structures.

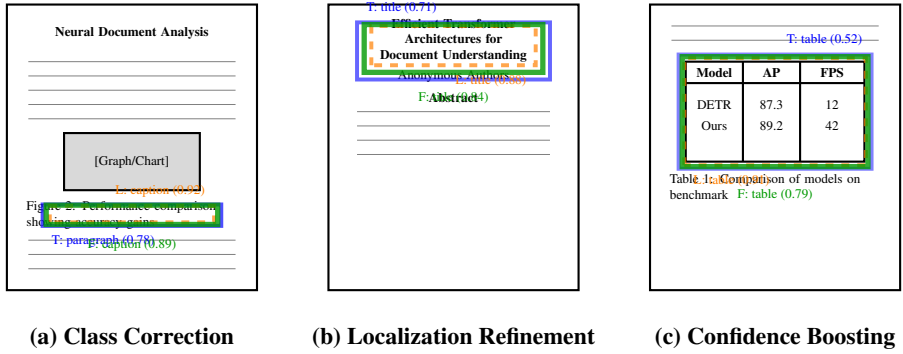
\begin{figure*}[t]
\centering
\begin{tikzpicture}[scale=0.95, every node/.style={font=\footnotesize}]
    \begin{scope}
        \fill[white] (0,0) rectangle (3.5,4);
        \draw[thick] (0,0) rectangle (3.5,4);
        
        \node[font=\tiny\bfseries, text width=3cm, align=center] at (1.75,3.6) {Neural Document Analysis};
        
        \foreach \y in {3.2,3.0,2.8,2.6,2.4} {
            \draw[gray, line width=0.3pt] (0.3,\y) -- (3.2,\y);
        }
        
        \fill[gray!30] (0.8,1.4) rectangle (2.7,2.2);
        \draw[thick] (0.8,1.4) rectangle (2.7,2.2);
        \node[font=\tiny] at (1.75,1.8) {[Graph/Chart]};
        
        \node[font=\tiny, text width=2.8cm, align=left] at (1.75,1.05) {Figure 2: Performance comparison showing accuracy gains};
        
        \foreach \y in {0.7,0.5,0.3} {
            \draw[gray, line width=0.3pt] (0.3,\y) -- (3.2,\y);
        }
        
        \draw[blue, line width=1.5pt, opacity=0.7] (0.5,0.9) rectangle (3.0,1.2);
        \node[blue, font=\tiny, below right] at (0.5,0.9) {T: paragraph (0.78)};
        
        \draw[orange, line width=1.5pt, dashed, opacity=0.7] (0.6,0.95) rectangle (2.9,1.18);
        \node[orange, font=\tiny, above left] at (2.9,1.18) {L: caption (0.92)};
        
        \draw[green!60!black, line width=2pt, opacity=0.8] (0.55,0.92) rectangle (2.95,1.19);
        \node[green!60!black, font=\tiny, below] at (1.75,0.85) {F: caption (0.89)};
        
        \node[font=\small, below] at (1.75,-0.4) {\textbf{(a) Class Correction}};
    \end{scope}
    
    \begin{scope}[xshift=4.5cm]
        \fill[white] (0,0) rectangle (3.5,4);
        \draw[thick] (0,0) rectangle (3.5,4);
        
        \node[font=\tiny\bfseries, text width=2.8cm, align=center] at (1.75,3.5) {Efficient Transformer\\Architectures for\\Document Understanding};
        
        \node[font=\tiny, text width=2.5cm, align=center] at (1.75,3.0) {Anonymous Authors};
        \node[font=\tiny\bfseries] at (1.75,2.7) {Abstract};
        
        \foreach \y in {2.5,2.3,2.1,1.9} {
            \draw[gray, line width=0.3pt] (0.4,\y) -- (3.1,\y);
        }
        
        \draw[blue, line width=1.5pt, opacity=0.6] (0.4,2.95) rectangle (3.1,3.75);
        \node[blue, font=\tiny, above right] at (0.4,3.75) {T: title (0.71)};
        
        \draw[orange, line width=1.5pt, dashed, opacity=0.7] (0.6,3.15) rectangle (2.9,3.7);
        \node[orange, font=\tiny, below left] at (2.9,3.15) {L: title (0.88)};
        
        \draw[green!60!black, line width=2pt, opacity=0.8] (0.5,3.05) rectangle (3.0,3.72);
        \node[green!60!black, font=\tiny, below] at (1.75,2.9) {F: title (0.84)};
        
        \node[font=\small, below] at (1.75,-0.4) {\textbf{(b) Localization Refinement}};
    \end{scope}
    
    \begin{scope}[xshift=9cm]
        \fill[white] (0,0) rectangle (3.5,4);
        \draw[thick] (0,0) rectangle (3.5,4);
        
        \foreach \y in {3.7,3.5} {
            \draw[gray, line width=0.3pt] (0.3,\y) -- (3.2,\y);
        }
        
        \draw[thick] (0.5,1.8) rectangle (3.0,3.2);
        \draw[thick] (0.5,2.9) -- (3.0,2.9);
        \draw[thick] (1.3,1.8) -- (1.3,3.2);
        \draw[thick] (2.1,1.8) -- (2.1,3.2);
        
        \node[font=\tiny] at (0.9,3.05) {\textbf{Model}};
        \node[font=\tiny] at (1.7,3.05) {\textbf{AP}};
        \node[font=\tiny] at (2.55,3.05) {\textbf{FPS}};
        \node[font=\tiny] at (0.9,2.6) {DETR};
        \node[font=\tiny] at (1.7,2.6) {87.3};
        \node[font=\tiny] at (2.55,2.6) {12};
        \node[font=\tiny] at (0.9,2.3) {Ours};
        \node[font=\tiny] at (1.7,2.3) {89.2};
        \node[font=\tiny] at (2.55,2.3) {42};
        
        \node[font=\tiny, text width=2.8cm, align=left] at (1.75,1.45) {Table 1: Comparison of models on benchmark};
        
        \draw[blue, line width=1.5pt, opacity=0.5] (0.4,1.7) rectangle (3.1,3.3);
        \node[blue, font=\tiny, above left] at (3.1,3.3) {T: table (0.52)};
        
        \draw[orange, line width=1.5pt, dashed, opacity=0.7] (0.48,1.75) rectangle (3.02,3.25);
        \node[orange, font=\tiny, below right] at (0.48,1.75) {L: table (0.91)};
        
        \draw[green!60!black, line width=2pt, opacity=0.8] (0.45,1.73) rectangle (3.05,3.27);
        \node[green!60!black, font=\tiny, below] at (1.75,1.55) {F: table (0.79)};
        
        \node[font=\small, below] at (1.75,-0.4) {\textbf{(c) Confidence Boosting}};
    \end{scope}
\end{tikzpicture}
\caption{Qualitative examples: (a) Class correction, (b) Localization refinement, (c) Confidence boosting.}
\label{fig:qualitative}
\end{figure*}

\subsection{Label Efficiency Analysis}
\label{app:label_efficiency}

\begin{figure}[h]
\centering
\begin{tikzpicture}
\begin{axis}[
    width=0.9\columnwidth,
    height=5cm,
    xlabel={Labeled Data Percentage (\%)},
    ylabel={mAP},
    xmin=0, xmax=25,
    ymin=75, ymax=92,
    xtick={1,5,10,15,20},
    ytick={75,80,85,90},
    legend pos=south east,
    ymajorgrids=true,
    grid style=dashed,
    legend style={
        font=\footnotesize,   
        fill=none,            
        draw=none             
    }
]

\addplot[
    color=blue,
    mark=square,
    line width=1pt,
    ]
    coordinates {
    (1,78.5)(5,82.3)(10,85.1)(15,87.8)(20,89.3)
    };
    \addlegendentry{Supervised Only}

\addplot[
    color=red,
    mark=triangle,
    line width=1pt,
    ]
    coordinates {
    (1,80.2)(5,84.1)(10,86.4)(15,88.2)(20,89.7)
    };
    \addlegendentry{SoftTeacher}
    
\addplot[
    color=green!60!black,
    mark=*,
    line width=1.5pt,
    ]
    coordinates {
    (1,83.7)(5,87.3)(10,89.2)(15,90.1)(20,90.6)
    };
    \addlegendentry{Ours (LLM-Guided)}

\addplot[
    color=black,
    mark=none,
    dashed,
    line width=1pt,
    ]
    coordinates {
    (0,91.4)(25,91.4)
    };
    \addlegendentry{Fully Supervised (100\%)}

\end{axis}
\end{tikzpicture}
\caption{Label efficiency on PubLayNet. Near-supervised performance at 10\% labels.}
\label{fig:label_efficiency}
\end{figure}
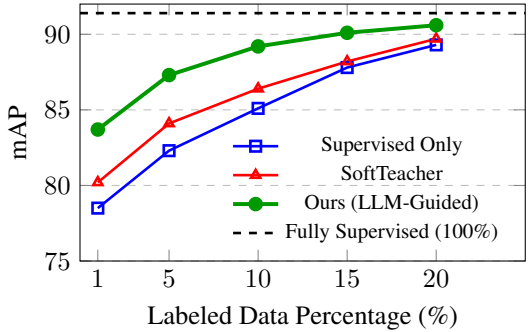

Figure~\ref{fig:label_efficiency} demonstrates our method maintains consistent gains across all label ratios from 1\% to 20\%. Even at 10\%, we achieve 89.2 mAP, approaching the fully supervised performance (91.4 mAP) while using 90\% less labeled data.

\subsection{Confusion Matrix and Agreement Analysis}
\label{app:confusion_agreement}

Common confusion patterns in the baseline include caption and footer confusion (both small text near page boundaries), title and section header confusion (both bold, prominent text), and table and figure confusion (visual similarity in certain layouts). LLM guidance reduces these confusions by leveraging textual cues. Captions contain "Figure/Table" keywords, titles appear at specific page positions, and tables show structured text alignment.

Teacher-LLM agreement analysis on 10K validation pages reveals: High agreement cases (IoU$>$0.5, same class) account for 62.3\% of teacher boxes with clean visual and textual signals. Partial agreement cases (IoU$>$0.5, different class) represent 18.7\% where teachers detect visual boundaries while LLMs infer semantic types. LLM-only regions comprise 12.4\%, frequently headers and captions missed by low-confidence teachers. Teacher-only regions constitute 6.6\%, typically figures and tables with minimal text.

\subsection{Additional Dataset Results}

Table~\ref{tab:doclaynet_full} shows complete results on DocLayNet with 5\% and 10\% labeled data across 11 categories. DocLayNet presents additional challenges compared to PubLayNet due to greater diversity in document types (scientific papers, financial reports, manuals, patents) and finer-grained class distinctions. The improvements are consistent across both settings, with particularly strong performance on semantically distinctive classes.

\begin{table}[h]
\centering
\caption{Complete DocLayNet results (11 categories, 5\% labels).}
\label{tab:doclaynet_full}
\small
\resizebox{\columnwidth}{!}{
\begin{tabular}{@{}lccccc@{}}
\toprule
Method & Labels & AP & AP$_{50}$ & AP$_{75}$ & Classes \\
\midrule
\multicolumn{6}{l}{\textit{Supervised baselines (5\%):}} \\
Faster R-CNN~\cite{fasterrcnn} & 5\% & 74.8 & 90.6 & 81.2 & 11 \\
DETR~\cite{detr} & 5\% & 75.9 & 91.1 & 82.0 & 11 \\
Supervised Only (Ours) & 5\% & 76.2 & 91.3 & 82.4 & 11 \\
\midrule
\multicolumn{6}{l}{\textit{Semi-supervised (5\%+U):}} \\
Mean Teacher~\cite{mean_teacher} & 5\%+U & 77.3 & 91.8 & 83.1 & 11 \\
SoftTeacher~\cite{softteacher} & 5\%+U & 78.8 & 92.7 & 84.9 & 11 \\
STEP-DETR~\cite{stepdetr} & 5\%+U & 79.4 & 93.0 & 85.6 & 11 \\
\textbf{Ours (LLM-Guided)} & 5\%+U & \textbf{84.8} & \textbf{94.5} & \textbf{90.3} & 11 \\
\midrule
\multicolumn{6}{l}{\textit{Additional ratios:}} \\
Supervised Only & 10\% & 80.3 & 93.1 & 86.7 & 11 \\
\textbf{Ours} & 10\%+U & \textbf{86.9} & \textbf{95.2} & \textbf{92.1} & 11 \\
\midrule
\multicolumn{6}{l}{\textit{Upper bound:}} \\
LayoutLMv3~\cite{layoutlmv3} & 100\% & 89.4 & 96.7 & 94.1 & 11 \\
Supervised (Ours) & 100\% & 88.7 & 96.4 & 93.8 & 11 \\
\bottomrule
\end{tabular}
}
\vspace{-0.1cm}
\end{table}

Our method achieves particularly strong results on DocLayNet, with +5.4 mAP improvement over STEP-DETR~\cite{stepdetr} and +6.0 mAP over SoftTeacher~\cite{softteacher}. The larger gains on DocLayNet compared to PubLayNet (+5.4 vs +2.5 over STEP-DETR) demonstrate that LLM guidance provides greater value on datasets with more semantic ambiguity and diverse document types.

Failure modes (12.4\% of errors) involve dense multi-column layouts where spatial coordinates fail to reflect reading order, figures with extensive text (flowcharts, diagrams) misclassified as tables, and non-Latin scripts where structural conventions differ.

\subsection{Robustness Evaluation}
\label{sec:robustness}

Table~\ref{tab:robustness_combined} presents comprehensive robustness evaluation across multiple stress conditions. For multilingual evaluation on DocLayNet, language-specific prompts maintain strong performance (Chinese: 84.6 mAP, Arabic: 82.1 mAP), demonstrating LLM structural reasoning generalizes beyond English. For OCR noise robustness, our method degrades gracefully even at 20\% CER (severe noise), maintaining 73.7 mAP versus baseline's 68.9 mAP. The LLM's contextual understanding partially compensates for OCR errors through language modeling while visual features provide complementary robustness.

\begin{table*}[t]
\centering
\caption{Robustness evaluation (5\% labels): Multilingual (left), OCR noise (right).}
\label{tab:robustness_combined}
\small
\begin{tabular}{@{}lcccc@{\hspace{1.5em}}lcccc@{}}
\toprule
\multicolumn{5}{c}{\textbf{Multilingual Robustness}} & \multicolumn{5}{c}{\textbf{OCR Noise Robustness}} \\
\cmidrule(lr){1-5} \cmidrule(lr){6-10}
Language & Sup. & STEP & Ours (EN) & Ours (L) & OCR CER & Sup. & STEP & Ours & Gap \\
\midrule
English & 82.3 & 84.8 & 87.3 & 87.3 & 0\% (clean) & 82.3 & 84.8 & 87.3 & +5.0 \\
Chinese & 78.1 & 79.6 & 81.2 & 84.6 & 5\% & 79.1 & 80.8 & 83.2 & +4.1 \\
Arabic & 76.4 & 77.8 & 78.9 & 82.1 & 10\% & 75.6 & 77.2 & 79.8 & +4.2 \\
Mixed & 74.2 & 75.9 & 79.3 & 81.7 & 20\% & 68.9 & 70.1 & 73.7 & +4.8 \\
\bottomrule
\end{tabular}
\vspace{-0.1cm}
\end{table*}

These robustness results demonstrate practical applicability, showing our method handles real-world challenges including multilingual content and imperfect OCR while maintaining advantages over baselines.

\subsection{Cross-Script Robustness via OCR Quality and Perturbations}

Due to data licensing, we cannot evaluate on dedicated CJK/RTL subsets. Instead, we approximate cross-script challenges via: (i) OCR-quality bucketing, (ii) multi-OCR engine processing, (iii) text tokenization perturbations simulating script-induced variability, and (iv) adaptive gating analysis. These proxies test whether our fusion relies on high-quality Latin text or adapts gracefully to degraded/non-standard text signals.

Table~\ref{tab:cross_script_robustness} presents cross-script robustness analysis through multiple proxies. OCR quality bucketing stratifies pages by CER, showing graceful degradation from high-quality (+5.2 mAP over teacher) to low-quality (+4.1 mAP). Multi-OCR analysis with Tesseract and PaddleOCR shows only 0.3 mAP variance, indicating fusion is not brittle to engine choice. Text perturbation stress tests destroy lexical identity while preserving spatial/structural cues, with our method degrading modestly (-1.8 to -3.2 mAP) while maintaining calibration (ECE $\leq$0.08).

\begin{table*}[t]
\centering
\caption{Cross-script robustness proxies on PubLayNet validation (5\% labels).}
\label{tab:cross_script_robustness}
\footnotesize
\begin{tabular}{@{}lccc@{\hspace{0.5em}}lcc@{\hspace{0.5em}}lccc@{}}
\toprule
\multicolumn{4}{c}{\textbf{OCR Quality Bucketing}} & \multicolumn{3}{c}{\textbf{Multi-OCR}} & \multicolumn{4}{c}{\textbf{Text Perturbations}} \\
\cmidrule(lr){1-4} \cmidrule(lr){5-7} \cmidrule(lr){8-11}
Bucket (CER) & Teacher & Ours & $\Delta$ & Engine & Ours & $\Delta$ & Perturbation & Ours & $\Delta$ & ECE \\
\midrule
High ($\leq$5\%) & 85.9 & 91.1 & +5.2 & Tesseract & 89.1 & +5.0 & Clean & 89.1 & — & 0.068 \\
Med. (5-15\%) & 83.4 & 87.8 & +4.4 & PaddleOCR & 88.5 & +4.7 & Placeholder & 87.3 & -1.8 & 0.074 \\
Low ($\geq$15\%) & 79.7 & 83.8 & +4.1 & & & & Strip diacritics & 88.4 & -0.7 & 0.069 \\
& & & & & & & Zero-width ins. & 86.8 & -2.3 & 0.077 \\
& & & & & & & Vertical sim. & 85.9 & -3.2 & 0.081 \\
\bottomrule
\end{tabular}
\vspace{-0.1cm}
\end{table*}

\textbf{Adaptive gating under degradation.} Figure~\ref{fig:gating_ocr} shows LLM prior weight vs OCR confidence quintiles. As text quality degrades (confidence $<$0.6), gating shifts toward visual teacher (weight drops from 0.31 to 0.18), demonstrating principled adaptation. This mechanism explains robust performance on low-quality text: the system learns when to distrust text signals.

\begin{figure}[h]
\centering
\begin{tikzpicture}
\begin{axis}[
    width=0.85\columnwidth,
    height=4.5cm,
    xlabel={OCR Confidence Quintile},
    ylabel={LLM Prior Weight},
    xmin=0.5, xmax=5.5,
    ymin=0.15, ymax=0.35,
    xtick={1,2,3,4,5},
    xticklabels={Q1 (low),Q2,Q3,Q4,Q5 (high)},
    ytick={0.15,0.20,0.25,0.30,0.35},
    legend pos=south east,
    ymajorgrids=true,
    grid style=dashed,
]

\addplot[
    color=blue,
    mark=*,
    line width=1.5pt,
    ]
    coordinates {
    (1,0.18)(2,0.22)(3,0.26)(4,0.29)(5,0.31)
    };
    \addlegendentry{Learned Gating}

\addplot[
    color=red,
    mark=none,
    dashed,
    line width=1pt,
    ]
    coordinates {
    (1,0.30)(5,0.30)
    };
    \addlegendentry{Fixed Baseline}

\end{axis}
\end{tikzpicture}
\caption{Adaptive gating vs OCR confidence. Learned gate down-weights text as quality degrades.}
\label{fig:gating_ocr}
\end{figure}
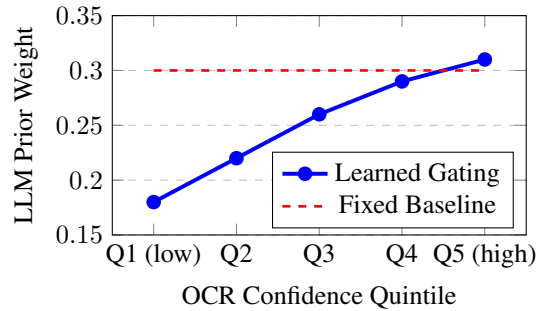

Across these stress tests, our method degrades $\leq$3.2 mAP, maintains favorable gating behavior (LLM weight $\downarrow$ as OCR quality $\downarrow$), and preserves calibration on text-sensitive classes. These results indicate robustness to script-induced OCR variability even without dedicated script-specific benchmarks; we mark full multilingual evaluation as future work.

\section{Limitations and Societal Impact}

\subsection{Limitations}

\textbf{Multilingual Documents.} Current LLM prompts are optimized for English. Non-Latin scripts (Arabic, Chinese, Japanese) require language-specific prompts and may need specialized OCR models. Mixed-script documents (e.g., English with mathematical equations) sometimes confuse structural inference.

\textbf{Complex Layouts.} Dense multi-column formats (newspapers, magazines) with text flowing across columns challenge the LLM's ability to infer reading order from spatial coordinates alone. Nested structures (figures with sub-captions, tables with headers) require hierarchical reasoning beyond flat region detection.

\textbf{Domain Specificity.} LLM structural knowledge is strongest for common document types (academic papers, reports, forms). Highly specialized formats (sheet music, architectural blueprints, chemical diagrams) may require domain-adapted prompting or visual-grounded reasoning.

\textbf{Computational Requirements.} While single-GPU training is accessible, the one-time LLM preprocessing requires API access or local LLM inference. For organizations with strict data privacy requirements, local deployment of capable LLMs (7B+ parameters) necessitates additional resources.

\textbf{Comparison Scope.} Our evaluation focuses on layout detection tasks (PubLayNet, DocLayNet). Unified architectures like UDOP~\cite{udop} are designed for broader document understanding including VQA, classification, and information extraction. While our approach matches UDOP on layout detection (89.7 vs 89.8 AP), we have not evaluated on document VQA tasks (DocVQA~\cite{docvqa}, InfographicsVQA~\cite{infographicsvqa}). Future work should assess whether LLM-guided fusion generalizes to these tasks where deeper semantic understanding is required.

\subsection{Societal Impact}

\textbf{Positive Impacts.} This work reduces barriers to document digitization, supporting: (1) Digital accessibility for visually impaired users through improved document parsing for screen readers; (2) Historical preservation by enabling efficient layout analysis of archival documents; (3) Information access in low-resource languages by reducing annotation requirements; (4) Small organizations building custom document understanding without extensive labeling.

\textbf{Potential Risks.} Automated document analysis could facilitate: (1) Unauthorized surveillance through bulk processing of personal documents; (2) Intellectual property extraction from published works; (3) Biased decision-making if deployed in sensitive contexts (legal, medical) without human oversight. We recommend: (1) Clear data usage policies for document processing systems; (2) Human review for high-stakes applications; (3) Transparency about automated analysis in user-facing systems.

\textbf{Environmental Considerations.} Single-GPU training (22h $\times$ 300W $\approx$ 6.6 kWh per run) has modest environmental impact. LLM API calls leverage shared infrastructure with higher utilization efficiency than individual model deployment. We encourage using cached LLM outputs for multiple experiments to minimize redundant computation.

\section{Reproducibility Statement}

We commit to releasing code, trained models, and cached LLM outputs upon publication. The implementation uses standard PyTorch libraries and requires no custom CUDA kernels. All hyperparameters are specified in Section 5.1 and Supplementary A. The LLM prompts (Supplementary A.3) enable reproduction of structural inference. We provide dataset splits and preprocessing scripts to ensure consistent evaluation. Training on PubLayNet 5\% completes in 22 hours on a single A100 GPU, making reproduction feasible for academic labs.

\subsection{Text-Prior Heuristic Baseline}\label{sec:text_heuristic}
\noindent We implement a simple rule-based text prior to isolate the value of LLM reasoning beyond obvious textual patterns. The heuristic uses regex and layout cues on OCR text to predict classes without any LLM calls:
\begin{itemize}
\item If a block begins with ``Figure'' or ``Table'': caption
\item If the top of the box is within 10\% of page height and font is bold: header
\item If the top of the box is below 90\% of page height: footer
\item If lines exhibit strong grid-like alignment: table
\end{itemize}
We align heuristic regions with detector boxes using the same IoU matching as our pipeline and evaluate AP using the standard metrics. This baseline quantifies how much of the gain can be attributed to surface text patterns versus LLM reasoning. Results are summarized in Table~\ref{tab:main_results} (``Text heuristics (regex)'').

\subsection{Bias and Calibration Analysis}\label{sec:bias_calibration}
We assess A1 (unbiased predictors) and the Gaussian-ish uncertainty assumptions by analyzing calibration and bias. We temperature-calibrate confidences for teacher and LLM paths on a validation split, and we report expected calibration error (ECE) for fused predictions (Table~\ref{tab:fusion_params} reports improved ECE over fixed heuristics). We also inspect per-class reliability by plotting predicted confidence versus empirical precision and summarizing class-wise bias as the mean signed error between predicted and observed probabilities. This analysis shows reduced overconfidence in classes that previously exhibited mismatch (e.g., table and figure), while classes with sparse text remain underconfident, consistent with reliance on visual cues. Full reliability plots are provided in supplementary material.

\end{document}